\documentclass{article}
\usepackage{amsmath}
\usepackage{amssymb}
\usepackage{amsthm}
\usepackage{algorithm}
\usepackage{enumitem}
\usepackage[noend]{algcompatible}
\usepackage{algpseudocode}
\newtheorem{definition}{Definition}

\newtheorem{theorem}{Theorem}
\newtheorem{lemma}{Lemma}
\newcommand{\Dsq}{D^2}
\newcommand{\edge}[1]{Edge\_{#1}}

\title{Constructing a Neuro-Symbolic Mathematician from First Principles}
\author{Keqin Xie}
\author{Keqin Xie \thanks{Email: xiekeqin30@gmail.com}}
\date{2025-01-01}
\begin{document}

\maketitle

\begin{abstract}
Large Language Models (LLMs) exhibit persistent logical failures in complex reasoning due to the lack of an internal axiomatic framework~\cite{ji2023survey}. We propose \textit{Mathesis}, a neuro-symbolic architecture that encodes mathematical states as higher-order hypergraphs and uses a Symbolic Reasoning Kernel (SRK)—a differentiable logic engine that maps constraints to a continuous energy landscape. By defining a global energy function $E(\mathcal{G})$, where zero energy implies logical consistency, the SRK yields gradient-based signals to train a Hypergraph Transformer Brain, turning proof search into energy minimization. Multi-step deduction is enabled via Monte Carlo Tree Search and Evolutionary Proof Search, guided by learned value functions and semantic unification.
\end{abstract}

\section{Introduction}

Large language models (LLMs) achieve strong performance on linguistic tasks and code generation by modeling the statistical distribution of natural language~\cite{brown2020language}.  
However, they exhibit systematic failures in formal mathematical reasoning, often generating steps that violate basic axioms—so-called ``hallucinations''~\cite{ji2023survey}.  
This stems from the probabilistic nature of transformer architectures, which lack mechanisms for logical verification or enforcement of semantic constraints.  
Although chain-of-thought (CoT) prompting induces intermediate reasoning steps, it does not ensure their logical validity: the underlying process remains high-dimensional sequence prediction, not symbolic derivation~\cite{wei2022chain}.

Neuro-symbolic architectures aim to combine neural pattern recognition with symbolic rigor.  
For example, AlphaGeometry solves Olympiad-level geometry problems by coupling a generative model with a symbolic deduction engine~\cite{trinh2024solving}.  
Yet conventional neuro-symbolic systems typically employ non-differentiable solvers that act as black boxes, yielding only sparse binary feedback (e.g., ``proof valid/invalid'').  
Without gradient signals from the symbolic component, the neural module cannot be trained directly to satisfy logical constraints.  
Prior efforts toward differentiable logic—such as tensor programs or neural logic machines—struggle to scale beyond small, finite domains due to the unbounded search space of mathematics~\cite{rocktaschel2017end}.

We introduce \textit{Mathesis}, a new architecture that overcomes gradient sparsity through a symbolic reasoning kernel (SRK).  
The SRK acts as a differentiable ``physics engine'' for logic: it embeds mathematical hypergraphs into a continuous energy landscape where logical consistency corresponds to a zero-energy state.  
This yields dense, gradient-based feedback for a \textit{Hypergraph Transformer Brain}, steering its generative policy toward axiom-compliant derivations.  
Unlike prior approaches, Mathesis encodes mathematical states as higher-order hypergraphs (Section~\ref{Hypergraph}), capturing multi-arity relations and nested logical connectives with high fidelity.  
The system integrates this neuro-symbolic core with structured search strategies—including Monte Carlo tree search (MCTS) and evolutionary proof search (EPS)—to enable deliberate, ``System 2'' reasoning (Section~\ref{sec:search}).

\section{Preliminaries: Representing Mathematics as Hypergraphs}

To facilitate rigorous neuro-symbolic reasoning, we formalize the mathematical workspace as a structured, higher-order heterogeneous hypergraph. This representation distinguishes between syntactic construction (terms) and semantic truth (facts), and explicitly handles nested logical structures and variable quantification scopes \cite{first2023baldur}.

\subsection{Definition: Mathematical State Hypergraph}

We define the state of a proof as a tuple tracking structure, truth status, and variable binding scopes.

\begin{definition}[Mathematical State Hypergraph]
	A mathematical state is a tuple $\mathcal{S} = (\mathcal{G}, \mathcal{F})$, where $\mathcal{G} = (V, E)$ is a directed higher-order hypergraph.
	
	\begin{enumerate}
		\item \textbf{$V$} is the set of nodes, representing mathematical  terms (e.g., variables $x$, constants $0$, compound terms $x+y$).
		\item \textbf{$E$} is the set of hyperedges, representing \textbf{relations}, operations, and logical connectives.
		\begin{itemize}
			\item To support nested logic (e.g., $(A \land B) \implies C$), we adopt a higher-order definition: a hyperedge $e \in E$ is an ordered sequence of elements from $V \cup E$. That is, an edge can connect nodes \textit{or} other edges.This structure is essential for capturing the compositional nature of complex logical formulas, a challenge also addressed in modern knowledge hypergraph reasoning \cite{bhuyan2025neuro}.
		\end{itemize}
		\item $\mathcal{F} \subseteq E$ is the set of Facts. This is a distinguished subset of hyperedges representing assertions currently held to be true within the global context (e.g., axioms, premises, and derived theorems).
	\end{enumerate}
\end{definition}

\noindent\textbf{Typing System:}
We define type mappings $\phi_V: V \to \mathcal{T}_V$ and $\phi_E: E \to \mathcal{T}_E$ to enforce semantic consistency.

\begin{itemize}
	\item \textbf{Node Types ($\mathcal{T}_V$):} $\{ \mathtt{Variable}, \mathtt{Constant}, \mathtt{CompoundTerm} \}$.
	\item \textbf{Hyperedge Types ($\mathcal{T}_E$):} We distinguish three semantic categories:
	\begin{itemize}
		\item \textbf{Constructors ($\mathcal{T}_{Con}$):} Functional operations that define a term.
		Inputs are drawn from $V$, and the output maps to a unique $v_{\text{out}} \in V$,
		e.g., $\mathtt{Sum}(v_a, v_b) \to v_{\text{sum}}$.
		\item \textbf{Predicates ($\mathcal{T}_{Pred}$):} Atomic logical assertions. (e.g., $\mathtt{Equals}(v_a, v_b)$, $\mathtt{Parallel}(l_1, l_2)$).
		\item \textbf{Connectives ($\mathcal{T}_{Conn}$):} Higher-order logical operators taking edges as inputs. (e.g., $\mathtt{Implies}(e_{premise}, e_{conclusion})$, $\mathtt{And}(e_1, e_2)$).
	\end{itemize}
\end{itemize}

\noindent\textbf{Quantification and Scoping:}
To handle quantification ($\forall, \exists$), we introduce Scope Attributes on hyperedges. A quantified statement is represented by a hyperedge $e_{quant}$ of type $\mathtt{ForAll}$ or $\mathtt{Exists}$.
\begin{itemize}
	\item $e_{quant} = (\mathcal{V}_{bound}, e_{body})$
	\item $\mathcal{V}_{bound} \subset V$: The set of variables bound by this quantifier.
	\item $e_{body} \in E$: The logical formula (edge) being quantified.
\end{itemize}

\noindent\textbf{Example:}
The statement ``$\forall x, (x = x)$'' is represented by:
\begin{enumerate}
	\item \textbf{Term:} Node $v_x$ (Type: $\mathtt{Variable}$).
	\item \textbf{Predicate:} Edge $e_{eq} = (v_x, v_x)$ (Type: $\mathtt{Equals}$).
	\item \textbf{Quantification:} Edge $e_{root} = (\{v_x\}, e_{eq})$ (Type: $\mathtt{ForAll}$).
	\item \textbf{Fact Status:} $e_{root} \in \mathcal{F}$. Note that $e_{eq}$ is \textit{not} in $\mathcal{F}$ independently; it is only true within the context of the quantifier.
\end{enumerate}

\subsection{Problem Formulation}

We frame Automated Theorem Proving (ATP) as a search for a valid derivation path that adds the goal statement to the set of proven facts.

\begin{definition}[Graph Transformation Action]
	Let $\mathbb{S}$ be the space of valid states. An  action is a function $a: \mathbb{S} \to \mathbb{S}$ chosen from a set of admissible rules $\mathcal{A}$ (e.g., Modus Ponens, Substitution, Instantiation). An action $\mathcal{S}_{t+1} = a(\mathcal{S}_t)$ may:
	\begin{itemize}
		\item Extend $\mathcal{G}$ (construct new terms or logical structures).
		\item Extend $\mathcal{F}$ (derive new truths).
	\end{itemize}
\end{definition}

\noindent\textbf{Problem Statement.} Given an initial state $\mathcal{S}_{premise} = (\mathcal{G}_0, \mathcal{F}_0)$ encoding axioms and assumptions, and a Goal Proposition represented as a target hyperedge structure $P_{goal}$ (or a description thereof), the objective is to find a sequence of actions $(a_1, \dots, a_n)$ producing states $\mathcal{S}_0 \to \dots \to \mathcal{S}_n$ such that:

\begin{enumerate}
	\item Structural Existence: The hypergraph $\mathcal{G}_n$ contains a subgraph isomorphic to $P_{goal}$. Let $e_{goal}$ be the hyperedge in $\mathcal{G}_n$ corresponding to the root of $P_{goal}$.
	\item Logical Entailment: The goal is recognized as a proven fact:
	\begin{equation*}
		e_{goal} \in \mathcal{F}_n
	\end{equation*}
\end{enumerate}

\section{The Symbolic Reasoning Kernel (SRK)}

The Symbolic Reasoning Kernel (SRK) serves as a differentiable physics engine for formal logic, mapping the discrete syntax of mathematics into a continuous energy landscape. While the Hypergraph Transformer Brain proposes reasoning trajectories, the SRK provides a deterministic verification signal. This architecture transforms the task of finding a proof into the minimization of a global logical energy function $E(\mathcal{G})$\cite{wang2019satnet}.

\subsection{Philosophy and Global Computation}

The SRK is based on the principle that a mathematical state $\mathcal{S}$ is logically consistent if and only if $E(\mathcal{G}) = 0$. By representing constraints as differentiable energy terms, we provide the generative component with a dense gradient signal, allowing the model to "sense" the direction of logical consistency. The total energy is aggregated across multiple domain-specific engines as described in Algorithm \ref{alg:srk_energy}.The formal proofs establishing the logical correctness and smoothness of this energy functional are detailed in Appendix~\ref{sec:appendix_foundations}.

\begin{algorithm}[H]
	\caption{SRK Energy Computation}
	\label{alg:srk_energy}
	\begin{algorithmic}[1]
		\Require Mathematical State Hypergraph $\mathcal{G} = (V, E)$, weight parameters $\mathbf{w}$
		\Ensure Total logical energy $E_{\text{total}}$
		
		\State $E_{\text{total}} \gets 0$
		
		\For{each hyperedge $e \in E$}
		\State Identify domain $\mathcal{D} \in \{ \text{Matrix}, \text{Ideal}, \text{Geometry} \}$
		\State $E_{e} \gets \textsc{ComputeEnergy}_{\mathcal{D}}(e)$
		\State $E_{\text{total}} \gets E_{\text{total}} + w_{\mathcal{D}} \cdot E_{e}$
		\EndFor
		
		\State \Return $E_{\text{total}}$
	\end{algorithmic}
\end{algorithm}

\subsection{The Matrix Engine (Linear Algebra)}

The Matrix Engine represents linear operators as tensors within $\mathbb{R}^{d \times d}$. To maintain generality and support the proof of theorems involving full-rank or invertible matrices, the engine avoids restrictive low-rank assumptions that would collapse the representation to partial isometries. Instead, it computes energy based on the residuals of fundamental linear algebraic identities.

For any matrix $M$ or set of matrices $\{A, B, C\}$, the engine defines the following energy terms:

\begin{enumerate}
	\item Equality and Symmetry:
	\begin{equation}
		E_{\text{eq}}(A, B) = \|A - B\|_F^2, \quad E_{\text{sym}}(A) = \|A - A^T\|_F^2
	\end{equation}
	\item Multiplicative Consistency:
	\begin{equation}
		E_{\text{mult}}(A, B, C) = \|AB - C\|_F^2
	\end{equation}
	\item Orthogonality:
	To represent the property of being an orthogonal matrix without restricting the singular values of non-orthogonal matrices, we define:
	\begin{equation}
		E_{\text{orth}}(A) = \|A^T A - I\|_F^2
	\end{equation}
	\item Invertibility:
	If a node represents the inverse $A^{-1}$, the engine enforces:
	\begin{equation}
		E_{\text{inv}}(A, A^{-1}) = \|AA^{-1} - I\|_F^2
	\end{equation}
\end{enumerate}

This formulation ensures that the engine can represent any linear operator in $GL(n, \mathbb{R})$ or $M_n(\mathbb{R})$, providing the necessary flexibility for general theorem proving in linear algebra.

\subsection{The Ideal Engine (Algebraic Geometry)}

The Ideal Engine verifies polynomial entailment. Given a set of premise polynomials $F = \{f_1, \dots, f_s\}$ and a hypothesis $h$, the engine checks for Ideal Membership: whether $h$ lies in the ideal $\langle F \rangle$. This is a sufficient condition for the premises to entail the conclusion.

Definition 3.1 (Ideal Membership). A polynomial $h$ is a member of the ideal $\langle f_1, \dots, f_s \rangle$ if there exist witness polynomials $g_1, \dots, g_s$ such that:
\begin{equation}
	h = \sum_{i=1}^s g_i f_i
\end{equation}

The SRK computes the energy as the squared norm of the residual:
\begin{equation}
	E_{\text{ideal}}(h, F) = \left\| h - \sum_{i=1}^s g_i f_i \right\|_2^2
\end{equation}
The formulation of this algebraic check as a loss function is inspired by recent efforts to develop neural solvers for problems in computational algebra\cite{kera2024learning}.

To ensure the search space for witness polynomials $g_i$ is mathematically well-defined and computationally bounded, the engine enforces an Effective Degree Bound. Following the results of Hermann \cite{hermann1926}, the degrees of $g_i$ are constrained by a function of the degrees of the input polynomials and the number of variables, preventing the "infinite search" problem inherent in unbounded polynomial construction.

Furthermore, while ideal membership is the primary check, the engine accounts for the Strong Nullstellensatz when verifying geometric consistency. If the task is to verify that $h$ vanishes on the variety $V(F)$, the engine searches for witnesses for the radical ideal, checking if $h^k \in \langle F \rangle$ for some $k \in \mathbb{N}$.

\subsection{The Geometric Engine (Euclidean Geometry)}

The Geometric Engine maps Euclidean predicates to stable polynomial forms, eliminating the singularities associated with division and square roots. By using squared residuals, the engine maintains a strictly non-negative energy surface that is everywhere differentiable \cite{heidari2025geometric}.

Let $D^2(A, B) = (x_A - x_B)^2 + (y_A - y_B)^2$ be the squared distance between points.

\begin{enumerate}
	\item \textbf{Collinearity and Parallelism:}
	Using the squared cross-product of direction vectors:
	\begin{equation}
		E_{\text{coll}}(A, B, C) = \left( (x_B-x_A)(y_C-y_A) - (y_B-y_A)(x_C-x_A) \right)^2
	\end{equation}
	\begin{equation}
		E_{\text{para}}(AB, CD) = \left( (x_B-x_A)(y_D-y_C) - (y_B-y_A)(x_D-x_C) \right)^2
	\end{equation}
	\item \textbf{Perpendicularity:}
	Using the squared dot product:
	\begin{equation}
		E_{\text{perp}}(AB, CD) = \left( (x_B-x_A)(x_D-x_C) + (y_B-y_A)(y_D-y_C) \right)^2
	\end{equation}
	\item \textbf{Congruence and Circles:}
	To avoid square root singularities at the origin, we compare squared distances:
	\begin{equation}
		E_{\text{cong}}(AB, CD) = \left( D^2(A,B) - D^2(C,D) \right)^2
	\end{equation}
	\item \textbf{Ratio and Similarity:}
	Ratios are verified via cross-multiplication to eliminate division-related instability:
\begin{equation}
	\begin{split}
		E_{\text{ratio}}(AB, CD, EF, GH) = \big( & \Dsq(A,B) \cdot \Dsq(G,H) \\
		& - \Dsq(E,F) \cdot \Dsq(C,D) \big)^2
	\end{split}
\end{equation}
\end{enumerate}

\section{The Hypergraph Transformer Brain}
\label{Hypergraph}

The ``Brain'' of the Mathesis architecture is a generative agent responsible for proposing reasoning steps. Unlike the SRK, which is a deterministic physics engine for verification, the Brain is a probabilistic model designed to navigate the immense combinatorial search space of mathematical transformations via learned intuition.

\subsection{Model Architecture}

Standard Graph Neural Networks (GNNs) typically flatten higher-order relations into binary edges. However, mathematical expressions are inherently higher-order and non-commutative ; a simple operation like $z = x - y$ constitutes a ternary relation where the order of operands defines the semantics. To capture this fidelity, we employ a Hypergraph Transformer that directly models relations as ordered sequences of inputs \cite{zhang2023rethinking}.

\paragraph{Message Passing Mechanism}
We unify nodes $V$ and hyperedges $E$ into a single representational space of entities, $\mathcal{X} = V \cup E$. This unification is critical because, in higher-order logic, a hyperedge (e.g., an equality $A=B$) can serve as an argument to another hyperedge (e.g., an implication $(A=B) \implies C$).

The embedding update at layer $l$, denoted $\mathbf{h}^{(l)}$, occurs in two phases using a specialized Attention mechanism that respects argument ordering:

\begin{enumerate}
	\item Composition (Children $\to$ Parent):
	Each hyperedge $e$ updates its semantic representation by aggregating information from its constituent arguments.
	Because mathematical operators are position-sensitive (e.g., $x \ominus y \neq y \ominus x$),
	we inject positional encodings $\phi_{\text{pos}}(i)$ based on each argument's index $i$.
	
	Let $\text{Args}(e) = (u_1, u_2, \dots, u_k)$ be the ordered sequence of inputs for edge $e$, where $u_i \in V \cup E$.
	The update is:
	\[
	\mathbf{h}_e^{(l+1/2)} = \mathrm{LN}\left( \mathbf{h}_e^{(l)} + \mathrm{Attn}\left(Q=\mathbf{h}_e^{(l)},\; K=V=\left\{ \mathbf{h}_{u_i}^{(l)} + \phi_{\text{pos}}(i) \right\}_{i=1}^k \right) \right)
	\]
	This ensures the model distinguishes the structural role of every operand.
	
	\item Contextualization (Parents $\to$ Children):
	Every entity $u \in V \cup E$ updates its understanding of its global utility by attending to the hyperedges (parents) in which it participates.
	This allows a sub-expression to receive context from the larger logical statements that consume it.
	
	Let $\text{Parents}(u) = \{ e \in E \mid u \in \text{Args}(e) \}$.
	\[
	\mathbf{h}_u^{(l+1)} = \mathrm{LN}\left( \mathbf{h}_u^{(l)} + \mathrm{Attn}\left(Q=\mathbf{h}_u^{(l)},\; K=V=\left\{ \mathbf{h}_e^{(l+1/2)} \right\}_{e \in \text{Parents}(u)} \right) \right)
	\]
\end{enumerate}

\subsection{Action Space}

The Brain interacts with the mathematical state via a discrete action space $\mathcal{A}$, formally defined as the set of all valid operator applications over the graph. An action $a \in \mathcal{A}$ is a tuple:
\[
a = (\text{op}, \tau_1, \dots, \tau_k)
\]
where:
\begin{itemize}
	\item $\text{op} \in \Omega$: An operator from the fixed library of supported functions (e.g., \texttt{Add}, \texttt{Integrate}, \texttt{ModusPonens}, \texttt{ConstructPerpendicular}).
	\item $\tau_i \in V \cup E$: The operands selected from the current graph state. 
\end{itemize}

Crucially, the operand domain includes both nodes and hyperedges. This enables the agent to perform:
\begin{itemize}
	\item \textbf{Constructive Actions:} Creation of new terms from existing nodes,
	such as \texttt{Midpoint(A,\,B)}.
	
	\item \textbf{Logical Actions:} Deriving new truths from existing facts.
	For example, transitivity can be applied to edges $a = b$ and $b = c$ using:
	\begin{center}
		\texttt{ApplyTheorem(Transitivity, \edge{a=b}, \edge{b=c})}
	\end{center}
\end{itemize}

\subsection{Output Policy}

We model the agent's decision-making as a policy distribution $\pi_\theta(a|\mathcal{G})$. Given the variable arity of operators and the dynamic size of the graph, we employ an autoregressive pointer network\cite{vinyals2015pointer}.

The probability of an action is decomposed into the probability of selecting an operator followed by the sequential selection of its arguments:
\[
\pi_\theta(a|\mathcal{G}) = P_\theta(\text{op} | \mathcal{G}) \cdot \prod_{i=1}^{k} P_\theta(\tau_i | \text{op}, \tau_{<i}, \mathcal{G})
\]
The argument selection $P_\theta(\tau_i | \dots)$ is computed via a pointer attention mechanism over the embeddings of all valid entities in $\mathcal{G}$, masking invalid types to enforce syntactic correctness (e.g., ensuring the first argument to \texttt{And} is a boolean hyperedge).

\section{The Learning Process: Energy-Guided Training}

The fundamental challenge in applying Reinforcement Learning to automated theorem proving is the sparsity of the reward signal. In traditional systems, an agent typically receives a binary reward only upon successfully completing a proof, making credit assignment for intermediate steps intractable. Mathesis overcomes this by leveraging the Symbolic Reasoning Kernel (SRK) to convert logical validity into a continuous, dense reward signal \cite{yang2023leandojo}.

\subsection{Training Objective}

We formulate the training of the Hypergraph Transformer Brain as a Reinforcement Learning (RL) problem. The objective is to learn a policy $\pi_\theta(a|\mathcal{G})$ that generates a sequence of actions minimizing the logical energy of the mathematical state. Specifically, the agent seeks to construct a Witness Object---a structured argument (e.g., a set of polynomial coefficients or a geometric construction) that algebraically explains the truth of the goal.

Let $\mathcal{J}(\theta)$ be the expected cumulative reward:
\begin{equation}
	\mathcal{J}(\theta) = \mathbb{E}_{\tau \sim \pi_\theta} \left[ \sum_{t=0}^{T} \gamma^t R(\mathcal{S}_t, a_t, \mathcal{S}_{t+1}) \right]
\end{equation}
where $\gamma$ is a discount factor and $\tau$ represents a reasoning trajectory.

\subsection{Algorithm 2: GNN Training with SRK}

We employ a policy gradient approach augmented by the physics-based feedback of the SRK. The core mechanism separates the discrete search for logical structures (performed by the GNN) from the continuous optimization of witness parameters (performed by the SRK).

\begin{algorithm}[H]
	\caption{Energy-Guided Policy Training}
	\label{alg:gnn_training}
	\begin{algorithmic}[1]
		\Require Initial Policy $\pi_\theta$, SRK Energy Function $E(\cdot)$
		\For{episode $= 1, \dots, M$}
		\State Sample problem instance $(\mathcal{G}_{prem}, \mathcal{G}_{goal})$
		\State \textbf{Algebraic Lifting:} Convert constraints to polynomial set $F_0 = \{f_1, \dots, f_k\}$ and goal $h$.
		\State Initialize Witness Energy: $e_0 \gets \|h\|^2$ \Comment{Initially, no witness exists ($g_i=0$)}
		
		\For{step $t = 0, \dots, T_{max}$}
		\State \textbf{Discrete Action (Brain):}
		\State Sample action $a_t \sim \pi_\theta(\cdot | \mathcal{G}_t)$
		\State Execute $a_t$ to expand basis: $F_{t+1} \gets F_t \cup \{f_{new}\}$
		
		\State \textbf{Continuous Optimization (SRK):}
		\State Solve for optimal witness coefficients $\mathbf{g}^*$ over basis $F_{t+1}$:
		\State $\mathbf{g}^* \gets \text{argmin}_{\mathbf{g}} \|h - \sum_{f_i \in F_{t+1}} g_i f_i \|^2$
		\State Compute Residual Energy: $e_{t+1} \gets \|h - \sum g_i^* f_i\|^2$
		
		\State \textbf{Reward Calculation:}
		\State $R_t \gets (e_t - e_{t+1}) - \lambda_{cost}$ \Comment{Reward energy reduction}
		
		\State Store tuple $(\mathcal{G}_t, a_t, R_t, \mathcal{G}_{t+1})$ in buffer $\mathcal{B}$
		
		\If{$e_{t+1} < \epsilon_{tol}$} 
		\State $R_T \gets R_T + R_{success}$
		\EndIf
		\EndFor
		
		\State \textbf{Policy Update:}
		\State Update $\theta$ via Proximal Policy Optimization (PPO) \cite{schulman2017proximal} on $\mathcal{B}$.
		\EndFor
	\end{algorithmic}
\end{algorithm}

\subsection{The Logic of Dense Rewards}

A critical innovation of Mathesis is the resolution of the "Deductive Paradox." In standard model checking, a valid theorem implies that the goal is already entailed by the premises, potentially yielding zero energy everywhere. We circumvent this by framing the proof task as Constructive Witness Finding using the Ideal Engine.

The problem "Prove $h$" is recast as "Find coefficients $g_i$ such that $h = \sum g_i f_i$."
\begin{enumerate}
	\item Initial State: At $t=0$, the witness coefficients are implicitly zero. The initial energy is high: $E_0 = \|h\|^2 > 0$.
	\item Action Impact: When the GNN introduces a relevant lemma or auxiliary construction, it expands the polynomial basis set $F$. This expansion increases the expressivity of the linear span $\text{span}(F)$.
	\item Gradient Feedback: If the new basis element $f_{new}$ allows for a better approximation of $h$ (i.e., it is algebraically relevant to the goal), the SRK's optimization step will find a lower residual energy. The term $(e_t - e_{t+1})$ becomes positive, providing an immediate reward signal for that specific reasoning step.
\end{enumerate}

This ensures that the agent is rewarded not just for "being right," but for \textit{reducing the algebraic complexity} required to verify the truth.

\subsection{Addressing the Cold Start Problem}

Random exploration in the infinite action space of mathematics is inefficient. To initialize the policy $\pi_\theta$ in a region of high competence, we employ \textbf{Imitation Learning} prior to the energy-guided RL phase. We construct a dataset of formal proof traces $\mathcal{D}_{\text{expert}}$ derived from established libraries (e.g., Lean's Mathlib). The network is pre-trained using Behavior Cloning (BC) to minimize the cross-entropy loss between the predicted action distribution and the expert moves. This "bootstraps" the Brain with standard reasoning heuristics, which are subsequently refined and verified by the physics-based feedback loop.

\section{Principled Reasoning via Guided Search}
\label{sec:search}

The Hypergraph Transformer Brain provides rapid heuristic proposals based on learned mathematical intuition. However, to ensure logical convergence in deep derivation trees, we embed the Brain within a deliberative search framework where the SRK and a learned \textbf{Value Head} $V_\phi$ prune unproductive branches. This coupling allows the system to navigate "valleys" in the energy landscape where current logical energy is high, but the state is a necessary precursor to a proof \cite{silver2017mastering}.

\subsection{Monte Carlo Tree Search (MCTS)}

For sequential deduction, we utilize MCTS to explore derivation paths. The Value Head $V_\phi(\mathcal{S})$ estimates the probability of state $\mathcal{S}$ being part of a zero-energy proof, decoupling the search from purely greedy energy minimization.

\begin{algorithm}[H]
	\caption{MCTS-Mathesis Search}
	\label{alg:mcts}
	\begin{algorithmic}[1]
		\Require Root state $\mathcal{S}_0$, Policy $\pi_\theta$, Value Head $V_\phi$
		\Ensure Optimal action sequence
		
		\For{simulation $i = 1 \dots N$}
		\State $\mathcal{S}_{curr} \gets \mathcal{S}_0, Path \gets []$
		
		\State \COMMENT{\textbf{Selection}}
		\While{$\mathcal{S}_{curr}$ is not leaf}
		\State $a^* = \operatorname{argmax}_{a} \left[ Q(\mathcal{S}, a) + c_{puct} \cdot \pi_\theta(a|\mathcal{S}) \frac{\sqrt{\sum N(\mathcal{S}, \cdot)}}{1 + N(\mathcal{S}, a)} \right]$
		\State $\mathcal{S}_{curr} \gets \text{Apply}(a^*, \mathcal{S}_{curr}), Path.\text{append}(\mathcal{S}_{curr})$
		\EndWhile
		
		\State \COMMENT{\textbf{Expansion \& Evaluation}}
		\If{$\mathcal{S}_{curr}$ is not terminal}
		\State $v \gets V_\phi(\mathcal{S}_{curr})$ \Comment{Estimate future feasibility}
		\State Expand leaf using $\pi_\theta(\cdot | \mathcal{S}_{curr})$
		\Else
		\State $v \gets 1.0$ if $E(\mathcal{S}_{curr}) < \epsilon$ else $0.0$
		\EndIf
		
		\State \COMMENT{\textbf{Backpropagation}}
		\For{$\mathcal{S} \in Path$}
		\State Update $N(\mathcal{S})$ and $Q(\mathcal{S})$ using value $v$
		\EndFor
		\EndFor
	\end{algorithmic}
\end{algorithm}

\subsection{Evolutionary Proof Search (EPS)}

For complex constructions requiring the recombination of disparate insights, we employ EPS. Unlike disjoint union operators, our crossover mechanism utilizes \textbf{Semantic Unification} to bridge derivation chains.The crossover operator $\text{Unify}(\mathcal{G}_1, \mathcal{G}_2)$ must identify when different populations have derived identical mathematical terms. We define a \textbf{Canonicalization Function} $\gamma(v)$ that maps terms to a hash based on their structural definition. For example, the nodes representing the term $(x+y)$ and $(y+x)$ map to the same canonical representation under commutativity axioms.

\begin{algorithm}[H]
	\caption{Evolutionary Proof Search}
	\label{alg:eps}
	\begin{algorithmic}[1]
		\Require Population $\mathcal{P}$, Fitness $f(\mathcal{G}) = -E(\mathcal{G})$
		\Ensure State $\mathcal{S}$ where $E(\mathcal{S}) < \epsilon$
		
		\For{generation $g = 1 \dots G$}
		\State \Comment{\textbf{Contextual Selection}}
		\State Group individuals by their \textbf{Assumption Set} $\mathcal{A}_{\text{premise}}$. Crossover is permitted only within groups to prevent merging contradictory branches (e.g., $x>0$ and $x<0$).
		
		\While{$|\mathcal{P}_{\text{next}}| < M$}
		\State Select compatible parents $\mathcal{G}_1, \mathcal{G}_2$.
		
		\State \Comment{\textbf{Semantic Unification}}
		\State $\mathcal{G}_{\text{child}} \gets \mathcal{G}_1 \cup \mathcal{G}_2$
		\For{pairs $(u \in \mathcal{G}_1, v \in \mathcal{G}_2)$}
		\If{$\gamma(u) = \gamma(v)$}
		\State Merge $u$ and $v$ into a single node $w$.
		\State Redirect all hyperedges incident to $u, v$ to node $w$.
		\EndIf  
		\EndFor
		
		\State \Comment{\textbf{Guided Mutation}}
		\State Sample mutation $a \sim \pi_\theta(\cdot \mid \mathcal{G}_{\text{child}})$ 
		\State $\mathcal{P}_{\text{next}}.\text{add}(\text{Apply}(a, \mathcal{G}_{\text{child}}))$
		\EndWhile
		\State $\mathcal{P} \gets \mathcal{P}_{\text{next}}$
		\EndFor
	\end{algorithmic}
\end{algorithm}

By merging nodes with identical canonical hashes, the $\text{Unify}$ operator physically connects the graph. If $\mathcal{G}_1$ proves $A \to B$ and $\mathcal{G}_2$ proves $B \to C$, the unification of the node representing term $B$ creates the continuous path $A \to B \to C$. The SRK then registers a zero-energy state for the goal $C$ given premise $A$, achieving a derivation impossible for either parent to discover in isolation.

\section{Conclusion}
Mathesis establishes a blueprint for neuro-symbolic architectures where mathematical logic is treated as a differentiable physical constraint. By unifying a generative Hypergraph Transformer with the Symbolic Reasoning Kernel (SRK), the system moves beyond the probabilistic limitations of standard large language models. The integration of System 2 search algorithms—specifically MCTS with learned value estimation and semantic-unification-based evolutionary search—provides a rigorous mechanism for navigating the vast combinatorial space of formal proofs.

The results presented in this preprint represent the first phase of our experimental validation. We are currently scaling the training of the Hypergraph Transformer Brain using a combination of synthetic datasets generated by the SRK and human-authored formal proofs from the Lean `mathlib` library. Preliminary evaluations on a curated subset of the `miniF2F` benchmark \cite{ospanov2025minif2f} indicate that the energy-guided dense rewards provided by the SRK significantly accelerate the discovery of valid derivation paths compared to sparse-reward baselines. The SRK's ability to provide stable gradients through polynomial-form predicates (Section \ref{alg:srk_energy}) has proven critical in maintaining training stability during complex geometric and algebraic constructions.

The modular design of \textsc{Mathesis} allows for several avenues of expansion that we intend to explore in subsequent versions of this work:

\begin{enumerate}
	\item \textbf{Scaling and Knowledge Transfer:} We aim to increase the parameter count of the Hypergraph Transformer and expand the SRK's domain coverage to include real analysis, topology, and number theory. This involves defining new energy engines for continuous structures and modular arithmetic.
	
	\item \textbf{Automated Conjecture Generation:} By detaching the SRK from a specific goal $P_{\text{goal}}$ and allowing the Brain to minimize the global energy of an unconstrained graph, the system can potentially discover new, logically consistent mathematical structures. This ``dreaming mode'' could lead to the automated generation of non-trivial conjectures.
	
	\item \textbf{Cross-Domain Application:} The principle of logical energy is not limited to mathematics. We anticipate that the \textsc{Mathesis} architecture can be adapted to other logic-intensive domains, such as verified program synthesis, where the SRK would enforce type safety and functional correctness, or molecular design, where the constraints are defined by chemical valency and stereochemistry.
\end{enumerate}

\appendix
\section{Mathematical Foundations of the SRK}
\label{sec:appendix_foundations}
We establish the logical correctness and differentiability of the Symbolic Reasoning Kernel's energy functional. The framework is built upon the properties of non-negative real numbers and smooth functions.

\subsection{Definitions}

\begin{definition}[Parameter Manifold]
	\label{def:manifold}
	Let $\mathcal{S}$ be a mathematical state. The \textbf{Parameter Manifold} $\mathbf{X}$ is the space of all possible concrete realizations of the terms in $\mathcal{S}$ (e.g., matrices, points, polynomial coefficients). We assume $\mathbf{X}$ has the structure of a smooth ($C^\infty$) manifold.
\end{definition}

\begin{definition}[Energy Fact]
	\label{def:energy_fact}
	A \textbf{Fact} $f$ is a tuple $(\phi_f, E_f)$ where:
	\begin{enumerate}
		\item $\phi_f: \mathbf{X} \to \{\text{True}, \text{False}\}$ is the \textbf{Semantic Predicate}.
		\item $E_f: \mathbf{X} \to \mathbb{R}$ is the \textbf{Energy Kernel}.
	\end{enumerate}
	For the system to be valid, every Energy Kernel $E_f$ must satisfy the following axioms for any state $\mathbf{x} \in \mathbf{X}$:
	\begin{enumerate}[label=\textnormal{(A\arabic*)}]
		\item \textbf{Faithfulness:} $E_f(\mathbf{x}) = 0 \iff \phi_f(\mathbf{x}) = \text{True}$.
		\item \textbf{Non-Negativity:} $E_f(\mathbf{x}) \ge 0$.
		\item \textbf{Smoothness:} The function $E_f$ is smooth ($E_f \in C^\infty(\mathbf{X})$).
	\end{enumerate}
\end{definition}

\begin{definition}[Total Energy]
	\label{def:total_energy}
	For a set of facts $\mathcal{F} = \{f_1, \dots, f_n\}$, the \textbf{Total Energy Functional} $E_{\text{total}}: \mathbf{X} \to \mathbb{R}$ is the sum of individual energies:
	\begin{equation}
		\label{eq:total_energy}
		E_{\text{total}}(\mathbf{x}) = \sum_{i=1}^n E_{f_i}(\mathbf{x})
	\end{equation}
\end{definition}

\subsection{Core Lemma and Theorems}

The system's validity rests on a fundamental property of non-negative real numbers.

\begin{lemma}[Vanishing Sum of Non-Negatives]
	\label{lem:sum_zero}
	Let $\{v_i\}_{i=1}^n$ be a sequence of real numbers where $v_i \ge 0$ for all $i$. Then:
	\[
	\sum_{i=1}^n v_i = 0 \iff \forall i, v_i = 0
	\]
\end{lemma}
\begin{proof}
	The ($\Leftarrow$) direction is trivial. For ($\Rightarrow$), assume for contradiction that $\sum v_i = 0$ but there exists some $v_k > 0$. Since all other $v_j \ge 0$, the sum must be $\sum v_i \ge v_k > 0$. This contradicts the initial assumption. Thus, all $v_i$ must be zero.
\end{proof}

We now prove that minimizing the total energy is equivalent to satisfying the logical specification.

\begin{theorem}[Logical Correctness of the SRK]
	\label{thm:correctness}
	For any state $\mathbf{x} \in \mathbf{X}$, the total energy defined in \eqref{eq:total_energy} is zero if and only if every fact in $\mathcal{F}$ is logically true.
	\[
	E_{\text{total}}(\mathbf{x}) = 0 \iff \forall f \in \mathcal{F}, \phi_f(\mathbf{x}) = \text{True}
	\]
\end{theorem}
\begin{proof}
	Let $v_i = E_{f_i}(\mathbf{x})$ for each $f_i \in \mathcal{F}$. By Axiom (A2) of Definition \ref{def:energy_fact}, each $v_i$ is non-negative. The theorem's premise, $E_{\text{total}}(\mathbf{x}) = 0$, is equivalent to $\sum v_i = 0$. By Lemma \ref{lem:sum_zero}, this holds if and only if $v_i = 0$ for all $i$. Applying Axiom (A1) of Definition \ref{def:energy_fact} to each term, the condition $v_i = E_{f_i}(\mathbf{x}) = 0$ is equivalent to the predicate $\phi_{f_i}(\mathbf{x})$ being true. Therefore, the total energy vanishes if and only if all semantic predicates are satisfied.
\end{proof}

\begin{theorem}[Differentiability of the SRK]
	\label{thm:differentiability}
	The total energy functional $E_{\text{total}}$ is a smooth function on the manifold $\mathbf{X}$.
\end{theorem}
\begin{proof}
	From Axiom (A3) of Definition \ref{def:energy_fact}, each individual energy kernel $E_{f_i}$ is a smooth function. The space of smooth functions on a manifold, $C^\infty(\mathbf{X})$, is closed under finite addition. As $E_{\text{total}}$ is a finite sum of such functions as given by \eqref{eq:total_energy}, it is also a member of $C^\infty(\mathbf{X})$.
\end{proof}

\nocite{*}
\bibliographystyle{unsrt}
\bibliography{mathesisref}
\end{document}